\newif\iffull
\fulltrue
\iffull
\documentclass[11pt,twoside]{article}
\usepackage{hyperref,xspace,fullpage,amsmath,amsthm,amssymb,amsfonts,boxedminipage}
\usepackage{paralist}
\usepackage{bm}
\usepackage{bbm}
\usepackage{array}
\usepackage{multirow}
\usepackage{times}
\usepackage{color}
\usepackage[numbers]{natbib}

\else
\documentclass{article}

\usepackage[final]{nips_2016}

\usepackage[utf8]{inputenc} 
\usepackage[T1]{fontenc}    
\usepackage{hyperref}       
\usepackage{url}            
\usepackage{booktabs}       
\usepackage{amsfonts}       
\usepackage{nicefrac}       
\usepackage{microtype}      

\usepackage{amsmath}
\usepackage{amsthm}
\fi
\usepackage{graphicx}
\usepackage{vfmacros}
\newcommand{\bff}{{\mathbf f}}

\newcommand{\bS}{{\mathbf S}}
\newcommand{\bV}{{\mathbf V}}
\newcommand{\para}[1]{\smallskip \noindent{\bf #1}}

\providecommand{\K}{{\mathcal K}}

\providecommand{\cH}{\mathcal{H}}
\newcommand{\ind}[1]{\mathbf{1}_{\{#1\}}}

\iffull
\title{Generalization of ERM in Stochastic Convex Optimization: \\The Dimension Strikes Back}
\else
\title{Generalization of ERM in Stochastic Convex Optimization: \\The Dimension Strikes Back\thanks{See \cite{Feldman:16erm} for the full version of this work. }}
\fi
\author{
  Vitaly Feldman \\
  IBM Research -- Almaden
}
\date{}
\begin{document}

\maketitle

\begin{abstract}
In stochastic convex optimization the goal is to minimize a convex function $F(x) \doteq \E_{\bff\sim D}[\bff(x)]$ over a convex set $\K \subset \R^d$ where $D$ is some unknown distribution and each $f(\cdot)$ in the support of $D$ is convex over $\K$. The optimization is commonly based on i.i.d.~samples $f^1,f^2,\ldots,f^n$ from $D$.
A standard approach to such problems is empirical risk minimization (ERM) that optimizes $F_S(x) \doteq \fr{n}\sum_{i\leq n} f^i(x)$. Here we consider the question of how many samples are necessary for ERM to succeed and the closely related question of uniform convergence of $F_S$ to $F$ over $\K$. We demonstrate that in the standard $\ell_p/\ell_q$ setting of Lipschitz-bounded functions over a $\K$ of bounded radius, ERM requires sample size that scales linearly with the dimension $d$. This nearly matches standard upper bounds and improves on $\Omega(\log d)$ dependence proved for $\ell_2/\ell_2$ setting in \cite{SSSSS:2009}. In stark contrast, these problems can be solved using dimension-independent number of samples for $\ell_2/\ell_2$ setting and $\log d$ dependence for $\ell_1/\ell_\infty$ setting using other approaches.

We further show that our lower bound applies even if the functions in the support of $D$ are smooth and efficiently computable and even if an $\ell_1$ regularization term is added. Finally, we demonstrate that for a more general class of bounded-range (but not Lipschitz-bounded) stochastic convex programs an infinite gap appears already in dimension 2.
\end{abstract}

\section{Introduction}
Numerous central problems in machine learning, statistics and operations research are special cases of stochastic optimization from i.i.d.~data samples. In this problem the goal is to optimize the value of the expected objective function $F(x) \doteq \E_{\bff\sim D}[\bff(x)]$ over some set $\K$ given i.i.d.~samples $f^1,f^2,\ldots,f^n$ of $\bff$. For example, in supervised learning the set $\K$ consists of hypothesis functions from $Z$ to $Y$ and each sample is an example described by a pair $(z,y) \in (Z,Y)$. For some fixed loss function $L:Y\times Y \rightarrow \R$, an example $(z,y)$ defines a function from $\K$ to $\R$ given by $f_{(z,y)}(h) = L(h(z),y)$. The goal is to find a hypothesis $h$ that (approximately) minimizes the expected loss relative to some distribution $P$ over examples: $\E_{(z,y)\sim P}[L(h(z),y)] = \E_{(z,y)\sim P}[f_{(z,y)}(h)]$.

Here we are interested in stochastic convex optimization (SCO) problems in which $\K$ is some convex subset of $\R^d$ and each function in the support of $D$ is convex over $\K$. The importance of this setting stems from the fact that such problems can be solved efficiently via a large variety of known techniques. Therefore in many applications even if the original optimization problem is not convex, it is replaced by a convex relaxation.



A classic and widely-used approach to solving stochastic optimization problems is empirical risk minimization (ERM) also referred to as stochastic average approximation (SAA) in the optimization literature. In this approach, given a set of samples $S=(f^1,f^2,\ldots,f^n)$ the empirical objective function: $F_S(x) \doteq \fr{n}\sum_{i\leq n} f^i(x)$ is optimized (sometimes with an additional regularization term such as $\lambda \|x\|^2$ for some $\lambda >0$). The question we address here is the number of samples required for this approach to work {\em distribution-independently}. More specifically, for some fixed convex body $\K$ and fixed set of convex functions $\F$ over $\K$, what is the smallest number of samples $n$ such that for every probability distribution $D$ supported on $\F$, any algorithm that minimizes $F_S$ given $n$ i.i.d.~samples from $D$ will produce an $\eps$-optimal solution $\hat{x}$ to the problem (namely, $F(\hat{x}) \leq \min_{x\in \K} F(x) + \eps$) with probability at least $1-\delta$? We will refer to this number as the sample complexity of ERM for $\eps$-optimizing $\F$ over $\K$  (we will fix $\delta =1/2$ for now).

The sample complexity of ERM for $\eps$-optimizing $\F$ over $\K$ is lower bounded by the sample complexity of $\eps$-optimizing $\F$ over $\K$, that is the number of samples that is necessary to find an $\eps$-optimal solution for any algorithm. On the other hand, it is upper bounded by the number of samples that ensures uniform convergence of $F_S$ to $F$. Namely, if with probability $\geq 1-\delta$, for all $x\in \K$, $|F_S(x) - F(x)| \leq \eps/2$ then, clearly, any algorithm based on ERM will succeed. As a result, ERM and uniform convergence are the primary tool for analysis of the sample complexity of learning problems and are the key subject of study in statistical learning theory. Fundamental results in VC theory imply that in some settings, such as binary classification and least-squares regression, uniform convergence is also a necessary condition for learnability (\eg \cite{Vapnik:98,Shalev-ShwartzBen-David:2014}) and therefore the three measures of sample complexity mentioned above nearly coincide.

In the context of stochastic convex optimization the study of sample complexity of ERM and uniform convergence was initiated in a groundbreaking work of Shalev-Shwartz, Shamir, Srebro and Sridharan \cite{SSSSS:2009}. They demonstrated that the relationships between these notions of sample complexity are substantially more delicate even in the most well-studied  settings of SCO. Specifically, let $\K$ be a unit $\ell_2$ ball and $\F$ be the set of all convex sub-differentiable functions with Lipschitz constant relative to $\ell_2$ bounded by 1 or, equivalently, $\| \nabla f(x)\|_2 \leq 1$ for all $x\in \K$. Then, known algorithm for SCO imply that sample complexity of this problem is $O(1/\eps^2)$ and often expressed as $1/\sqrt{n}$ rate of convergence (\eg \citep{Nemirovski:2009,Shalev-ShwartzBen-David:2014}). On the other hand, Shalev-Shwartz \etal \cite{SSSSS:2009} show\footnote{The dependence on $d$ is not stated explicitly but follows  immediately from their analysis.} that the sample complexity of ERM for solving this problem with $\eps = 1/2$ is $\Omega(\log d)$. The only known upper bound for sample complexity of ERM is $\tilde O(d/\eps^2)$ and relies only on the uniform convergence of Lipschitz-bounded functions \cite{ShapiroNemirovsky:05,SSSSS:2009}.


As can seen from this discussion, the work of Shalev-Shwartz \etal \cite{SSSSS:2009} still leaves a major gap between known bounds on sample complexity of ERM (and also uniform convergence) for this basic Lipschitz-bounded $\ell_2/\ell_2$ setup. Another natural question is whether the gap is present in the popular $\ell_1/\ell_\infty$ setup. In this setup $\K$ is a unit $\ell_1$ ball (or in some cases a simplex) and $\| \nabla f(x)\|_\infty \leq 1$ for all $x\in \K$. The sample complexity of SCO in this setup is $\theta(\log d/\eps^2)$ (\eg \citep{Nemirovski:2009,Shalev-ShwartzBen-David:2014}) and therefore, even an appropriately modified lower bound in \cite{SSSSS:2009}, does not imply any gap. More generally, the choice of norm can have a major impact on the relationship between these sample complexities and hence needs to be treated carefully. For example, for (the reversed) $\ell_\infty/\ell_1$ setting the sample complexity of the problem is $\theta(d/\eps^2)$ (\eg \cite{FeldmanGV:15}) and nearly coincides with the number of samples sufficient for uniform convergence.


\subsection{Overview of Results}
In this work we substantially strengthen the lower bound in \cite{SSSSS:2009} proving that a linear dependence on the dimension $d$ is necessary for ERM (and, consequently, uniform convergence). We then extend the lower bound to all $\ell_p/\ell_q$ setups and examine several related questions. Finally, we examine a more general setting of bounded-range SCO (that is $|f(x)| \leq 1$ for all $x\in \K$). While the sample complexity of this setting is still low (for example $\tilde O(1/\eps^2)$ when $\K$ is an $\ell_2$ ball) and efficient algorithms are known, we show that ERM might require an infinite number of samples already for $d=2$.

Our work implies that in SCO, even optimization algorithms that exactly minimize the empirical objective function can produce solutions with generalization error that is much larger than the generalization error of solutions obtained via some standard approaches. Another, somewhat counterintuitive, conclusion from our lower bounds is that, from the point of view of generalization of ERM and uniform convergence, convexity does not reduce the sample complexity in the worst case.

\para{Basic construction:} Our basic construction is fairly simple and its analysis is inspired by the technique in \citep{SSSSS:2009}. It is based on functions of the form $\max\{1/2,\max_{v \in V} \la v ,x \ra\}$. Note that the maximum operator preserves both convexity and Lipschitz bound (relative to any norm). See Figure \ref{fig:basic} for an illustration of such function for $d=2$.
\begin{figure}[h]
\begin{center}
\includegraphics[width=0.6\textwidth]{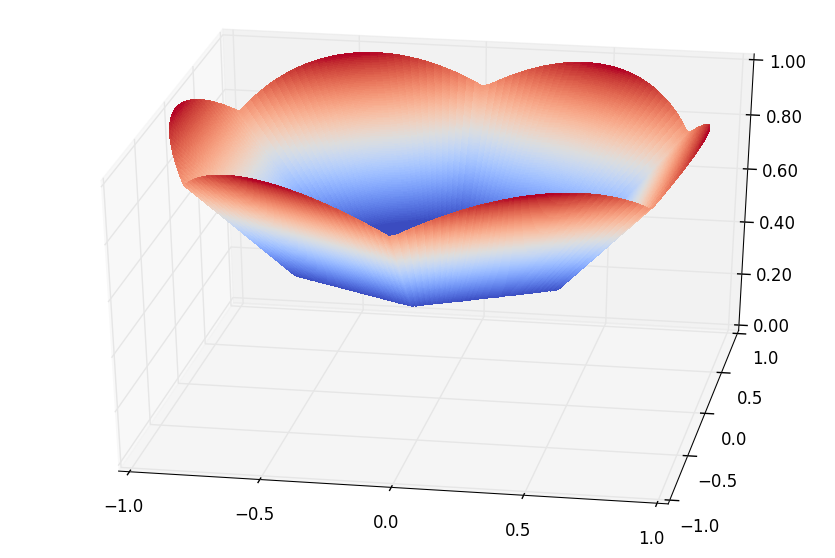}
\end{center}
\caption{Basic construction for $d=2$.}
\label{fig:basic}
\end{figure}

The distribution over the sets $V$ that define such functions is uniform over all subsets of some set of vectors $W$ of size $2^{d/6}$ such that for any two district $u,v \in W$, $\la u,v\ra \leq 1/2$. Equivalently, each element of $W$ is included in $V$ with probability $1/2$ independently of other elements in $W$. This implies that if the number of samples is less than $d/6$ then, with probability $>1/2$, at least one of the vectors in $W$ (say $w$) will not be observed in any of the samples.  This implies that $F_S$ can be minimized while maximizing $\la w ,x \ra$ (the maximum over the unit $\ell_2$ ball is $w$). Note that a function randomly chosen from our distribution includes the term $\la w ,x \ra$ in the maximum operator with probability $1/2$. Therefore the value of the expected function $F$ at $w$ is $3/4$ whereas the minimum of $F$ is $1/2$. In particular, there exists an ERM algorithm with generalization error of at least $1/4$. The details of the construction appear in Sec.~\ref{sec:lower-non-smooth} and Thm.~\ref{thm:l2-max} gives the formal statement of the lower bound. We also show that, by scaling the construction appropriately, we can obtain the same lower bound for any $\ell_p/\ell_q$ setup with $1/p+1/q=1$ (see Thm.~\ref{thm:non-smooth-lp}).

\para{Low complexity construction:} The basic construction relies on functions that require $2^{d/6}$ bits to describe and exponential time to compute. Most application of SCO use efficiently computable functions and therefore it is natural to ask whether the lower bound still holds for such functions. To answer this question we describe a construction based on a set of functions where each function requires just $\log d$ bits to describe (there are at most $d/2$ functions in the support of the distribution) and each function can be computed in $O(d)$ time. To achieve this we will use $W$ that consists of (scaled) codewords of an asymptotically good and efficiently computable binary error-correcting code \cite{Justesen:72,Spielman:96}. The functions are defined in a similar way but the additional structure of the code allows to use at most $d/2$ subsets of $W$ to define the functions. Further details of the construction appear in Section \ref{sec:complexity}.

 \para{Smoothness:} The use of maximum operator results in functions that are highly non-smooth (that is, their gradient is not Lipschitz-bounded) whereas the construction in \cite{SSSSS:2009} uses smooth functions. Smoothness plays a crucial role in many algorithms for convex optimization (see \cite{Bubeck15} for examples). It reduces the sample complexity of SCO in $\ell_2/\ell_2$ setup to $O(1/\eps)$ when the smoothness parameter is a constant (\eg \cite{Nemirovski:2009,Shalev-ShwartzBen-David:2014}).
Therefore it is natural to ask whether our strong lower bound holds for smooth functions as well. We describe a modification of our construction that proves a similar lower bound in the smooth case (with generalization error of $1/128$). The main idea is to replace each linear function $\la v,x\ra$ with some smooth function $\nu(\la v,x\ra)$ guaranteing that for different vectors $v^1,v^2 \in W$ and every $x \in \K$, only one of $\nu(\la v^1,x\ra)$ and $\nu(\la v^2,x\ra)$ can be non-zero. This allows to easily control the smoothness of $\max_{v \in V} \nu(\la v ,x \ra)$. See Figure \ref{fig:smooth} for an illustration of a function on which the construction is based (for $d=2$). The details of this construction appear in Sec.~\ref{sec:smooth} and the formal statement in Thm.~\ref{thm:l2-smooth}.
\begin{figure}[h]
\begin{center}
\includegraphics[width=\textwidth]{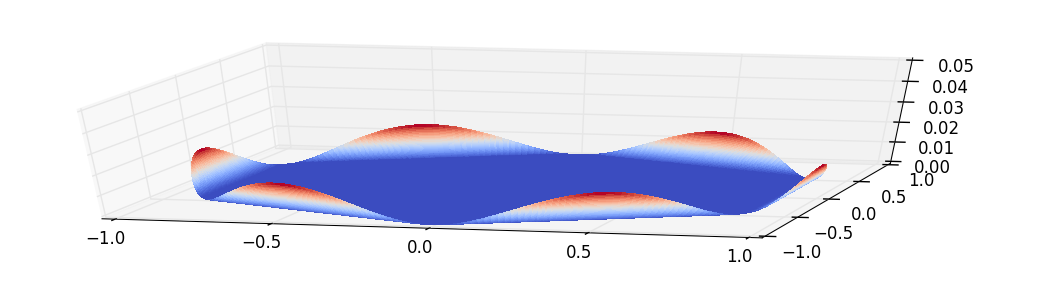}
\end{center}
\caption{Construction using 1-smooth functions for $d=2$.}
\label{fig:smooth}
\end{figure}

\para{$\ell_1$-regularization:} Another important contribution in  \cite{SSSSS:2009} is the demonstration of the important role that strong convexity plays for generalization in SCO: Minimization of $F_S(x) + \lambda R(x)$ ensures that ERM will have low generalization error whenever $R(x)$ is strongly convex (for a sufficiently large $\lambda$). This result is based on the proof that ERM of a strongly convex Lipschitz function is  {\em uniform replace-one} stable and the connection between such stability and generalization showed in \citep{BousquettE02} (see also \cite{ShwartzSSS10} for a detailed treatment of the relationship between generalization and stability).
It is natural to ask whether other approaches to regularization will ensure generalization. We demonstrate that for the commonly used $\ell_1$ regularization the answer is negative. We prove this using a simple modification of our lower bound construction: We shift the functions to the positive orthant where the regularization terms $\lambda \|x\|_1$ is just a linear function. We then subtract this linear function from each function in our construction, thereby balancing the regularization (while maintaining convexity and Lipschitz-boundedness). The details of this construction appear in Sec.~\ref{sec:l2-l1-reg} (see Thm.~\ref{thm:l2-l1}).

\para{Dependence on accuracy:} For simplicity and convenience we have ignored the dependence on the accuracy $\eps$, Lipschitz bound $L$ and radius $R$ of $\K$ in our lower bounds. It is easy to see, that this more general setting can be reduced to the case we consider here (Lipschitz bound and radius are equal to 1) with accuracy parameter $\eps'= \eps/(LR)$. We generalize our lower bound to this setting and prove that $\Omega(d/\eps'^2)$ samples are necessary for uniform convergence and $\Omega(d/\eps')$ samples are necessary for generalization of ERM. Note that the upper bound on the sample complexity of these settings is $\tilde O(d/\eps'^2)$ and therefore the dependence on $\eps'$ in our lower bound does not match the upper bound  for ERM. Resolving this gap or even proving any $\omega(d/\eps'+1/\eps'^2)$ lower bound is an interesting open problem. Additional details can be found \iffull in Section \ref{sec:eps}\else in the full version\fi.

\para{Bounded-range SCO:} Finally, 
we consider a more general class of bounded-range convex functions 
Note that the Lipschitz bound of 1 and the bound of 1 on the radius of $\K$ imply a bound of 1 on the range (up to a constant shift which does not affect the optimization problem). While this setting is not as well-studied, efficient algorithms for it are known. For example, the online algorithm in a recent work of Rakhlin and Sridharan \cite{RakhlinS15} together with standard online-to-batch conversion arguments \cite{Cesa-BianchiCG04}, imply that the sample complexity of this problem is $\tilde{O}(1/\eps^2)$ for any $\K$ that is an $\ell_2$ ball (of any radius). For general convex bodies $\K$, the problems can be solved via random walk-based approaches \cite{BelloniLNR15,FeldmanGV:15} or an adaptation of the center-of-gravity method given in \cite{FeldmanGV:15}. 
Here we show that for this setting ERM might completely fail already for $\K$ being the unit 2-dimensional ball. The construction is based on ideas similar to those we used in the smooth case and is formally described in \iffull Sec.~\ref{sec:non-lipschitz}. See Figure \ref{fig:nl} for an illustration of a function used in this construction.
\begin{figure}[h]
\begin{center}
\includegraphics[width=0.7\textwidth]{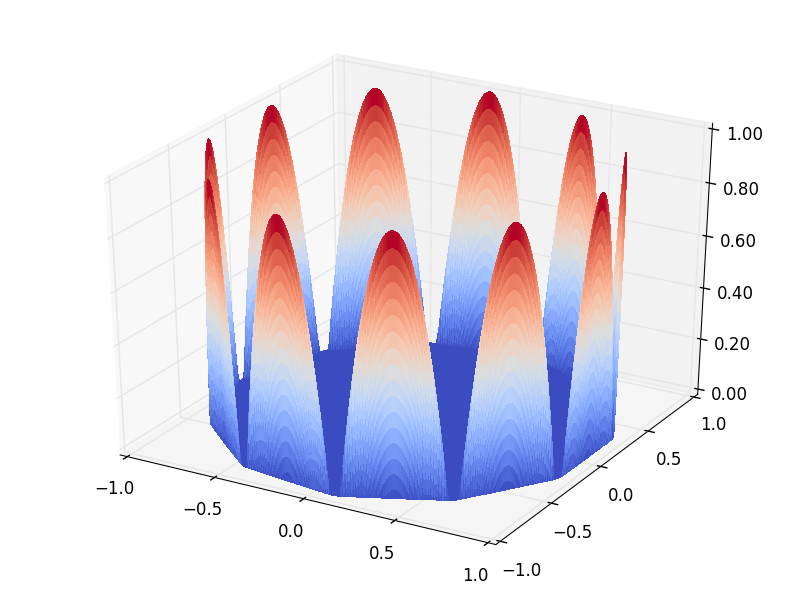}
\end{center}
\caption{Construction using non-Lipschitz convex functions with range in $[0,1]$.}
\label{fig:nl}
\end{figure}
\else in the full version.
\fi
\section{Preliminaries}
\label{sec:prelims}
For an integer $n\geq 1$ let $[n]\doteq \{1,\ldots, n\}$. 
Random variables are denoted by bold
letters, e.g., $\bff$.
Given $p\in [1,\infty]$ we denote the ball of radius $R>0$ in $\ell_p$ norm  by
$\B_p^d(R)$, and the unit ball by $\B_p^d$.

For a convex body (i.e., compact convex set with nonempty interior) $\K\subseteq \R^d$, we consider problems of the form
\begin{equation*} \label{StochOpt}
\min_\K(F_D) \doteq \min_{ x\in \K}\left\{F_D(x)\doteq\E_{\bff \sim D}[\bff(x)]\right\},
\end{equation*}
where $\bff$ is a random variable defined over some set of convex, sub-differentiable functions $\F$
on $\K$ and distributed according to some unknown probability distribution $D$. We denote $F^*=\min_\K(F_D)$. For an approximation parameter $\eps>0$ the goal is to find $x\in{\cal K}$ such that $F_D(x) \leq F^* +\eps$  and we call any such $x$ an {\em$\eps$-optimal solution}. 
 For an $n$-tuple of functions $S =(f^1,\ldots,f^n)$ we denote by $F_S \doteq \frac{1}{n} \sum_{i \in [n]} f^i$.

We say that a point $\hat{x}$ is an empirical risk minimum for an $n$-tuple $S$ of functions over $\K$, if $F_S(\hat{x}) = \min_\K(F_S)$. In some cases there are many points that minimize $F_S$ and in this case we refer to a specific algorithm that selects one of the minimums of $F_S$ as an empirical risk minimizer. To make this explicit we refer to the output of such a minimizer by $\hat{x}(S)$ .

Given $x\in\K$, and a convex function $f$ we denote by $\nabla f(x)\in \partial f(x)$ an arbitrary selection of a
subgradient. Let us make a brief reminder of some important classes of convex functions.
Let $p \in [1,\infty]$ and $q=p_\ast \doteq 1/(1-1/p)$. We say that a subdifferentiable convex function $f:\K\to\R$ is in the class
\begin{itemize}
\item $\F(\K,B)$ of $B$-bounded-range functions if for all $x\in \K$, $|f(x)| \leq B$.
\item $\F_{p}^0(\K,L)$ of $L$-Lipschitz continuous functions
w.r.t.~$\ell_p$, if for all $x,y\in\K$,
 $|f(x)-f(y)|\leq L\|x-y\|_p$;
\item $\F_{p}^1(\K,\sigma)$ of functions with $\sigma$-Lipschitz continuous
gradient w.r.t.~$\ell_p$, if for all $x,y\in\K$, $\|\nabla f(x)-\nabla f(y)\|_q\leq \sigma \|x-y\|_p$.
\eat{this implies
\begin{equation*}\label{smooth_dif_ineq}
f(y) \leq f(x) +\la\nabla f(x),y-x \ra +\frac{\sigma}{2}\|y-x\|_p^2.
\end{equation*}}
\eat{\item ${\cal S}_{p}(\K,\kappa)$ of $\kappa$-strongly convex functions w.r.t.~$\ell_p$, if for all $x,y\in\K$
\begin{equation} \label{str_cvx_dif_ineq}
f(y) \geq f(x) +\la\nabla f(x),y-x \ra +\frac{\kappa}{2}\|y-x\|_p^2.
\end{equation}}
\end{itemize}
We will omit $p$ from the notation when $p=2$. \iffull\else Omitted proofs can be found in the full version \cite{Feldman:16erm}\fi.

\section{Lower Bounds for Lipschitz-Bounded SCO}
\label{sec:lower-lipschitz}
In this section we present our main lower bounds for SCO of Lipschitz-bounded convex functions. For comparison purposes we start by formally stating some known bounds on sample complexity of solving such problems.
The following uniform convergence bounds can be easily derived from the standard covering number argument (\eg \cite{ShapiroNemirovsky:05,SSSSS:2009})
\begin{thm}
For $p \in [1,\infty]$, let $\K \subseteq \B_p^d(R)$ and let $D$ be any distribution supported on functions $L$-Lipschitz on $\K$ relative to $\ell_p$ (not necessarily convex). Then, for every $\eps,\delta >0$ and $n \geq n_1 = O\lp\frac{d \cdot (LR)^2 \cdot \log(dLR/(\eps\delta))}{\eps^2}\rp$
$$\pr_{\bS \sim D^n}\lb \exists x \in \K,\ \left| F_D(x) - F_\bS(x) \right|  \geq \eps \rb \leq \delta .$$
\end{thm}
The following upper bounds on sample complexity of  Lipschitz-bounded  SCO can be obtained from several known algorithms \cite{Nemirovski:2009,SSSSS:2009} (see \cite{Shalev-ShwartzBen-David:2014} for a textbook exposition for $p=2$).
\begin{thm}
For $p \in [1,2]$, let $\K \subseteq \B_p^d(R)$. Then,
there is an algorithm $\A_p$ that given $\eps,\delta >0$ and $n=n_p(d,R,L,\eps,\delta)$ i.i.d.~samples from any distribution $D$ supported on $\F_p^0(\K,L)$, outputs an $\eps$-optimal solution to $F_D$ over $\K$ with probability $\geq 1-\delta$.
For $p\in(1,2]$, $n_p = O((LR/\eps)^2 \cdot \log(1/\delta))$ and for $p=1$, $n_p = O((LR/\eps)^2 \cdot \log d \cdot \log(1/\delta))$.
\end{thm}
Stronger results are known under additional assumptions on smoothness and/or strong convexity (\eg \cite{Nemirovski:2009,RakhlinSS12,ShamirZ13,BachM13}).

\subsection{Non-smooth construction}
\label{sec:lower-non-smooth}
We will start with a simpler lower bound for non-smooth functions. For simplicity, we will also restrict $R=L =1$. Lower bounds for the general setting can be easily obtained from this case by scaling the domain and desired accuracy\iffull (see Thm.~\ref{thm:l2-max-eps} for additional details)\fi.

We will need a set of vectors $W \subseteq \on^d$ with the following property: for any distinct $w^1,w^2 \in W$, $\la w^1, w^2 \ra \leq d/2$. The Chernoff bound together with a standard packing argument imply that there exists a set $W$ with this property of size $\geq e^{d/8} \geq 2^{d/6}$.

For any subset $V$ of $W$ we define a function \equ{g_V(x) \doteq \max\{1/2, \max_{w \in V} \la \bar{w}, x \ra \} ,\label{eq:def-gv}} where $\bar{w} \doteq w/\|w\| = w/\sqrt{d}$. See Figure \ref{fig:basic} for an illustration. We first observe that $g_V$ is convex and $1$-Lipschitz (relative to $\ell_2$). This immediately follows from $\la \bar{w}, x \ra$ being convex and $1$-Lipschitz for every $w$ and $g_V$
being the maximum of convex and $1$-Lipschitz functions.
\begin{thm}
\label{thm:l2-max}
Let $\K = \B_2^d$ and we define $\cH_2 \doteq \{g_V \cond V \subseteq W\}$ for $g_V$ defined in eq.~\eqref{eq:def-gv}. Let $D$ be the uniform distribution over $\cH_2$. Then for $n \leq d/6 $ and every set of samples $S$ there exists an ERM $\hat{x}(S)$ such that
$$\pr_{\bS \sim D^n}\lb  F_{D}(\hat{x}(\bS)) - F^* \geq 1/4 \rb > 1/2. $$
\end{thm}
\begin{proof}
We start by observing that the uniform distribution over $\cH_2$ is equivalent to picking the function $g_{\bV}$ where $\bV$ is obtained by including every element of $W$ with probability $1/2$ randomly and independently of all other elements. Further,  by the properties of $W$, for every $w \in W$, and $V\subseteq W$, $g_V(\bar{w}) = 1$ if $w\in V$ and $g_V(\bar{w}) = 1/2$ otherwise. For $g_\bV$ chosen randomly with respect to $D$, we have that $w \in \bV$ with probability exactly $1/2$. This implies that $F_{D}(\bar{w}) = 3/4$.

Let $\bS= (g_{\bV_1}, \ldots, g_{\bV_n})$ be the random samples. Observe that $\min_{\K}(F_{\bS}) = 1/2$ and $F^* =\min_{\K}(F_D) = 1/2$ (the minimum is achieved at the origin $\bar{0}$). Now, if $\bigcup_{i \in [n]}\bV_i \neq W$ then let $\hat{x}(\bS) \doteq \bar{w}$ for any $w\in W\setminus \bigcup_{i \in [n]}\bV_i$. Otherwise $\hat{x}(\bS)$ is defined to be the origin $\bar{0}$. Then by the property of $\cH_2$ mentioned above, we have that for all $i$, $g_{\bV_i}(\hat{x}(\bS)) = 1/2$ and hence $F_\bS(\hat{x}(\bS)) = 1/2$. This means that $\hat{x}(\bS)$ is a minimizer of $F_\bS$.

Combining these statements, we get that, if $\bigcup_{i \in [n]}\bV_i \neq W$ then there exists an ERM $\hat{x}(\bS)$ such that $F_\bS(\hat{x}(\bS)) = \min_\K(F_\bS)$ and $F_D(\hat{x}(\bS)) - F^* = 1/4$. Therefore to prove the claim it suffices to show that for $n \leq d/6$ we have that $$\pr_{\bS\sim D^n}\lb\bigcup_{i \in [n]}\bV_i \neq W\rb >\fr{2} .$$ This easily follows from observing that for the uniform distribution over subsets of $W$, for every $w\in W$, $$\pr_{\bS\sim D^n}\lb w \in \bigcup_{i \in [n]}\bV_i\rb = 1-2^{-n}$$ and this event is independent from the inclusion of other elements in $\bigcup_{i \in [n]}\bV_i$. Therefore
$$\pr_{\bS\sim D^n}\lb\bigcup_{i \in [n]}\bV_i = W\rb = \lp1-2^{-n}\rp^{|W|} \leq e^{-2^{-n} \cdot 2^{d/6}} \leq e^{-1} < \fr{2} .$$
\end{proof}

\iffull
\begin{rem}
In our construction there is a different ERM algorithm that does solve the problem (and generalizes well). For example, the algorithm that always outputs the origin $\bar{0}$. Therefore it is natural to ask whether the same lower bound holds when there exists a unique minimizer. Shalev-Shwartz \etal \cite{SSSSS:2009} show that their lower bound construction can be slightly modified to ensure that the minimizer is unique while still having large generalization error. An analogous modification appears to be much harder to analyze in our construction and it is unclear to us how to ensure uniqueness in our strong lower bounds. A further question in this direction is whether it is possible to construct a distribution for which the empirical minimizer with large generalization error is unique and its value is noticeably (at least by $1/\poly(d)$) smaller than the value of $F_S$ at any point $x$ that generalizes well. Such distribution would imply that the solutions that ``overfits" can be found easily (for example, in a polynomial number of iterations of the gradient descent).
\end{rem}
\fi

\paragraph{Other $\ell_p$ norms:} We now observe that exactly the same approach can be used to extend this lower bound to $\ell_p/\ell_q$ setting. Specifically, for $p \in [1,\infty]$ and $q=p_\ast$ we define
$$g_{p,V}(x) \doteq \max\left\{\fr{2}, \max_{w \in V} \frac{\la w, x \ra}{d^{1/q}} \right\} .$$ It is easy to see that for every $V \subseteq W$, $g_{q,V} \in \F_{p}^0(\B_p^d,1)$. We can now use the same argument as before with the appropriate normalization factor for points in $\B_p^d$. Namely, instead of $\bar{w}$ for $w\in W$ we consider the values of the minimized functions at $w/d^{1/p} \in \B_p^d$. This gives the following generalization of Thm.~\ref{thm:l2-max}.
\begin{thm}
\label{thm:non-smooth-lp}
For every $p \in [1,\infty]$ let $\K = \B_p^d$ and we define $\cH_p \doteq \{g_{p,V} \cond V \subseteq W\}$ and let $D$ be the uniform distribution over $\cH_p$. Then for $n \leq d/6$ and every set of samples $S$ there exists an ERM $\hat{x}(S)$ such that
$$\pr_{\bS \sim D^n}\lb  F_{D}(\hat{x}(\bS)) - F^*  \geq 1/4 \rb > 1/2. $$
\end{thm}

\subsection{Smoothness does not help}
\label{sec:smooth}
We now extend the lower bound to smooth functions. We will for simplicity restrict our attention to $\ell_2$ but analogous modifications can be made for other $\ell_p$ norms. The functions $g_V$ that we used in the construction use two maximum operators each of which introduces non-smoothness. To deal with maximum with $1/2$ we simply replace the function $\max\{1/2,\la \bar{w},x\ra \}$ with a quadratically smoothed version (in the same way as hinge loss is sometimes replaced with modified Huber loss). To deal with the maximum over all $w \in V$, we show that it is possible to ensure that individual components do not ``interact". That is, at every point $x$, the value, gradient and Hessian of at most one component function are non-zero (value, vector and matrix, respectively).  This ensures that maximum becomes addition and Lipschitz/smoothness constants can be upper-bounded easily.

Formally, we define $$\nu(a) \doteq  \left\{\arr{ll}{ 0 & \mbox{if }  a \leq 0 \\ a^2 & \mbox{otherwise.}} \right.$$
Now, for $V \subseteq W$, we define \equ{ h_V(x) \doteq \sum_{w\in V} \nu(\la \bar{w},x \ra-7/8) \label{eq:def-hv}.}
See Figure \ref{fig:smooth} for an illustration.  We first prove that $h_V$ is $1/4$-Lipschitz and 1-smooth.
\begin{lem}
\label{lem:smoothness}
For every $V\subseteq W$ and $h_V$ defined in eq.~\eqref{eq:def-hv} we have  $h_V\in \F_{2}^0(\B_2^d,1/4) \cap \F_{2}^1(\B_2^d,1)$.
\end{lem}\iffull
\begin{proof}
It is easy to see that $\nu(\la \bar{w},x \ra-7/8)$ is convex for every $w$ and hence $h_V$ is convex. Next we observe that for every point $x\in \B_2^d$, there is at most one $w\in W$ such that $\la \bar{w}, x \ra > 7/8$. If $\la \bar{w}, x \ra > 7/8$ then $\|\bar{w} - x\|^2 = \|\bar{w}\|^2 + \|x\|^2 - 2 \la \bar{w}, x \ra  < 1+1 - 2(7/8) = 1/4$. On the other hand, by the properties of $W$, for distinct $w^1,w^2$ we have that $\|\bar{w}^1 - \bar{w}^2\|^2 = 2 - 2\la \bar{w}^1, \bar{w}^2 \ra \geq 1$. Combining these bounds on distances we obtain that if we assume that $\la \bar{w}^1, x \ra > 7/8$ and $\la \bar{w}^2, x \ra > 7/8$ then we obtain a contradiction $$\|\bar{w}^1 - \bar{w}^2\| \leq \|\bar{w}^1 - x\| + \|\bar{w}^2 - x\| < 1.$$

From here we can conclude that $$\nabla h_V(x) = \left\{\arr{ll}{ 2 (\la \bar{w}, x \ra - 7/8) \cdot \bar{w}\ & \mbox{if } \exists w\in V, \ \la \bar{w}, x \ra > 7/8 \\ 0 & \mbox{otherwise}}. \right.$$
This immediately implies that $\|\nabla h_V(x)\| \leq 1/4$ and hence $h_V$ is $1/4$-Lipschitz.

We now prove smoothness. Given two points $x,y \in \B_2^d$ we consider two cases. First the simpler case when there is at most one $w \in V$ such that either
$\la \bar{w}, x \ra > 7/8$ or $\la \bar{w}, y \ra > 7/8$. In this case $\nabla h_V(x) = \nabla \nu(\la \bar{w},x \ra-7/8)$ and $\nabla h_V(y) = \nabla \nu(\la \bar{w},y \ra-7/8)$. This implies that the 1-smoothness condition is implied by 1-smoothness of $\nu(\la \bar{w}, \cdot \ra-7/8)$. That is one can easily verify that $\|\nabla h_V(x)-\nabla h_V(y)\| \leq \|x-y\|$.

Next we consider the case where for $x$ there is $w^1\in V$ such that $\la \bar{w}^1, x \ra > 7/8$, for $y$ there is $w^2 \in V$ such that $\la \bar{w}^2, y \ra > 7/8$ and $w^1 \neq w^2$. Then there exists a point $z \in \B_2^d$ on the line connecting $x$ and $y$ such that $\la \bar{w}^1, z \ra \leq 7/8$ and $\la \bar{w}^2, z \ra \leq 7/8$. Clearly, $\|x - y\| = \|x - z\| +  \|z - y\|$. On the other hand, by the analysis of the previous case we have that $\|\nabla h_V(x)-\nabla h_V(z)\| \leq \|x-z\|$ and $\|\nabla h_V(z)-\nabla h_V(y)\| \leq \|z-y\|$. Combining these inequalities we obtain that
$$\|\nabla h_V(x)-\nabla h_V(y)\| \leq \|\nabla h_V(x)-\nabla h_V(z)\| + \|\nabla h_V(z)-\nabla h_V(y)\| \leq \|x - z\| +  \|z - y\| = \|x - y\| .$$
\end{proof} \fi

From here we can use the proof approach from Thm.~\ref{thm:l2-max} but with $h_V$ in place of $g_V$.
\begin{thm}
\label{thm:l2-smooth}
Let $\K = \B_2^d$ and we define $\cH \doteq \{h_V \cond V \subseteq W\}$ for $h_V$ defined in eq.~\eqref{eq:def-hv}. Let $D$ be the uniform distribution over $\cH$. Then for $n \leq d/6 $ and every set of samples $S$ there exists an ERM $\hat{x}(S)$ such that
$$\pr_{\bS \sim D^n}\lb F_{D}(\hat{x}(\bS)) - F^*  \geq 1/128 \rb > 1/2. $$
\end{thm}\iffull
\begin{proof}
Let $\bS= (h_{\bV_1}, \ldots, h_{\bV_n})$ be the random samples. As before we first note that $\min_{\K}(F_{\bS}) = 0$ and $F^* = 0$. Further, for every $w \in W$, $h_V(\bar{w}) = 1/64$ if $w\in V$ and $h_V(\bar{w}) = 0$ otherwise. Hence $F_D(\bar{w}) = 1/128$. Now, if $\bigcup_{i \in [n]}\bV_i \neq W$ then let $\hat{x}(\bS) \doteq \bar{w}$ for some $w\in W\setminus \bigcup_{i \in [n]}\bV_i$. Then for all $i$, $h_{\bV_i}(\hat{x}(\bS)) = 0$ and hence $F_\bS(\hat{x}(\bS)) = 0$. This means that $\hat{x}(\bS)$ is a minimizer of $F_\bS$ and $F_D(\hat{x}(\bS)) - F^* = 1/128$.

Now, exactly as in Thm.~\ref{thm:l2-max}, we can conclude that $\bigcup_{i \in [n]}\bV_i \neq W$ with probability $>1/2$.
\end{proof} \fi
\subsection{$\ell_1$ Regularization does not help}
\label{sec:l2-l1-reg}
Next we show that the lower bound holds even with an additional $\ell_1$ regularization term $\lambda \|x\|$ for positive $\lambda \leq 1/\sqrt{d}$. (Note that if $\lambda > 1/\sqrt{d}$ then the resulting program is no longer 1-Lipschitz relative to $\ell_2$. Any constant $\lambda$ can be allowed for $\ell_1/\ell_\infty$ setup). To achieve this we shift the construction to the positive orthant (that is $x$ such that $x_i \geq 0$ for all $i\in [d]$). In this orthant the subgradient of the regularization term is simply $\lambda \bar{1}$ where $\bar{1}$ is the all $1$'s vector. We can add a linear term to each function in our distribution that balances this term thereby reducing the analysis to non-regularized case.
More formally, we define the following family of functions. For $V \subseteq W$, \equ{h_V^\lambda(x) \doteq h_V(x-\bar{1}/\sqrt{d}) - \lambda \la \bar{1},x\ra \label{eq:def-h-lambda}.} Note that over $\B_2^d(2)$, $h_V^\lambda(x)$ is $L$-Lipschitz for $L \leq 2 (2-7/8) + \lambda \sqrt{d} \leq 9/4$.
We now state and prove this formally.

\begin{thm}
\label{thm:l2-l1}
Let $\K = \B_2^d(2)$ and for a given $\lambda \in (0,1/\sqrt{d}]$, we define $\cH^\lambda \doteq \{h_V^\lambda \cond V \subseteq W\}$ for $h_V^\lambda$ defined in eq.~\eqref{eq:def-h-lambda}. Let $D$ be the uniform distribution over $\cH^\lambda$. Then for $n \leq d/6$ and every set of samples $S$ there exists $\hat{x}(S)$ such that
\begin{itemize}
\item $F_S(\hat{x}(S)) = \min_{x\in \K}(F_S(x) + \lambda \|x\|_1)$;
\item $\pr_{\bS \sim D^n}\lb  F_{D}(\hat{x}(\bS)) - F^*   \geq 1/128 \rb > 1/2. $
\end{itemize}
\end{thm} \iffull
\begin{proof}
Let $\bS= (h_{\bV_1}^\lambda, \ldots, h_{\bV_n}^\lambda)$ be the random samples. We first note that $F^* = F_D(\bar{0}) = 0$ and
\alequn{ \min_{x\in\K}(F_\bS(x) + \lambda \|x\|_1) &= \min_{x\in\K}\lp \sum_{i\in [n]} h_{\bV_i}\lp x-\frac{\bar{1}}{\sqrt{d}}\rp - \lambda \la \bar{1},x\ra + \lambda \|x\|_1\rp \\&\geq \min_{x\in\K}\lp \sum_{i\in [n]} h_{\bV_i}\lp x-\frac{\bar{1}}{\sqrt{d}}\rp\rp \geq 0 .}

Further, for every $w \in W$, $\bar{w} + \bar{1}/\sqrt{d}$ is in the positive orthant and in $\B_2^d(2)$. Hence $h_V^\lambda(\bar{w}+ \bar{1}/\sqrt{d}) = h_V(\bar{w})$. We can therefore apply the analysis from Thm.~\ref{thm:l2-smooth} to obtain the claim.
\end{proof} \fi
\iffull
\subsection{Dependence on $\eps$}
\label{sec:eps}
We now briefly consider the dependence of our lower bound on the desired accuracy. Note that the upper bound for uniform convergence scales as $\tilde O(d/\eps^2)$. 

We first observe that our construction implies a lower bound of $\Omega(d/\eps^2)$ for uniform convergence nearly matching the upper bound (we do this for the simpler non-smooth $\ell_2$ setting but the same applies to other setting we consider).
\begin{thm}
\label{thm:l2-max-uniform}
Let $\K = \B_2^d$ and we define $\cH_2 \doteq \{g_V \cond V \subseteq W\}$ for $g_V$ defined in eq.~\eqref{eq:def-gv}. Let $D$ be the uniform distribution over $\cH_2$. Then for any $\eps > 0$ and $n \leq n_1 = \Omega(d/\eps^2)$ and every set of samples $S$ there exists a point $\hat{x}(S)$ such that
$$\pr_{\bS \sim D^n}\lb  F_{D}(\hat{x}(\bS)) - F_{\bS}(\hat{x}(\bS)) \geq \eps \rb > 1/2. $$
\end{thm}
\begin{proof}
For every $w \in W$, $$F_\bS(\bar{w}) = \fr{n} \sum_{i\in[n]} g_{\bV_i}(\bar{w}) = \fr{2} + \fr{2n} \sum_{i\in[n]}\ind{w\in \bV_i},$$
where $\ind{w\in \bV_i}$ is the indicator variable of $w$ being in $\bV_i$. If for some $w$,  $\fr{2n} \sum_{i\in[n]}\ind{w\in \bV_i} \geq 1/4+\eps$ then we will obtain a point $\bar{w}$ that violates the uniform convergence by $\eps$.  For every $w$, $\sum_{i\in[n]}\ind{w\in \bV_i}$ is distributed according to the binomial distribution. Using a standard approximation of the partial binomial sum up to $(1/2-2\eps)n$, we obtain that for some constant $c > 0$, the probability that this sum is $\geq 1/2+2\eps$ is at least
$$ \fr{2^n} \cdot \frac{1}{\sqrt{8n(1/4-\eps^2)}} \cdot \lp \fr{2} +2\eps \rp^{(1/2 +2\eps)n} \cdot \lp \fr{2} - 2\eps \rp^{(1/2 -2\eps)n} \geq 2^{-c n \eps^2}.$$
Now, using independence between different $w \in W$, we can conclude that, for $n \leq  d/(6c\eps^2)$, the probability that there exists  $w$ for which uniform convergence is violated is at least
$$1- \lp1-2^{- c n\eps^2}\rp^{|W|} \geq 1-e^{-2^{- c n\eps^2} \cdot 2^{d/6}} \geq 1- e^{-1} > \fr{2} .$$
\end{proof}

A natural question is whether the $d/\eps^2$ dependence also holds for ERM. We could not answer it and prove only a weaker $\Omega(d/\eps)$ lower bound. For completeness, we also make this statement for general radius $R$ and Lipschitz bound $L$.
\begin{thm}
\label{thm:l2-max-eps}
For $L,R >0$ and $\eps \in (0, LR/4)$, let $\K = \B_2^d(R)$ and we define $\cH_2 \doteq \{L \cdot g_V \cond V \subseteq W\} \subseteq \F^0(\B_2^d(R),L)$ for $g_V$ defined in eq.~\eqref{eq:def-gv}. We define the random variable $\bV_{\alpha}$ as a random subset of $W$ obtained by including each element of $W$ with probability $\alpha \doteq 2\eps/(LR)$ randomly and independently. Let $D_\alpha$ be the probability distribution of the random variable $g_{\bV_\alpha}$. Then for $n \leq d/32 \cdot LR/\eps$ and every set of samples $S$ there exists an ERM $\hat{x}(S)$ such that
$$\pr_{\bS \sim D_\alpha^n}\lb  F_{D_\alpha}(\hat{x}(\bS)) - F^* \geq \eps \rb > 1/2. $$
\end{thm}
\begin{proof}
By the same argument as in the proof of Thm.~\ref{thm:l2-max} we have that: For every $w \in W$, and $V\subseteq W$, $L \cdot g_V(R\bar{w}) = LR$ if $w\in V$ and $L \cdot g_V(R\bar{w}) = LR/2$ otherwise. For $g_\bV$ chosen randomly with respect to $D_\alpha$, we have that $w \in \bV$ with probability $2\eps/(LR)$. This implies that $F_{D_\alpha}(R\bar{w}) = LR/2 + \eps$. Similarly, $\min_{\K}(F_{\bS}) = LR/2$ and $F^* =\min_{\K}(F_{D_\alpha}) = LR/2$.

Therefore, if $\bigcup_{i \in [n]}\bV_i \neq W$ then there exists an ERM $\hat{x}(\bS)$ such that $F_\bS(\hat{x}(\bS)) = \min_\K(F_\bS)$ and $F_{D_\alpha}(\hat{x}(\bS)) - F^* = \eps$. For the distribution $D_\alpha$ and every $w\in W$, $$\pr_{\bS\sim D_\alpha^n}\lb w \in \bigcup_{i \in [n]}\bV_i\rb = 1-(1-\alpha)^{n} \leq 1-e^{-2\alpha n}$$ and this event is independent from the inclusion of other elements in $\bigcup_{i \in [n]}\bV_i$ (where we used that $1-\alpha \geq e^{-2\alpha}$ for $\alpha < 1/2$).
Therefore
$$\pr_{\bS\sim D_\alpha^n}\lb\bigcup_{i \in [n]}\bV_i = W\rb = \lp1- e^{-2\alpha n}\rp^{|W|} \leq e^{-e^{-2\alpha n} \cdot e^{d/8}} \leq e^{-1} < \fr{2} .$$
\end{proof}
\fi
\section{Lower Bound for Low-Complexity Functions}
\label{sec:complexity}
We will now demonstrate that our lower bounds hold even if one restricts the attention to functions that can be computed efficiently (in time polynomial in $d$). For this purpose we will rely on known constructions of binary linear error-correcting codes.
We describe the construction for non-smooth $\ell_2/\ell_2$ setting but analogous versions of other constructions can be obtained in the same way.

We start by briefly providing the necessary background about binary codes. For two vectors $w^1,w^2 \in \pmi^d$ let $\#_{\neq}(w^1,w^2)$ denote the Hamming distance between the two vectors. We say that a mapping $G:\pmi^k \rar \pmi^d$ is a $[d,k,r,T]$ binary error-correcting code if $G$ has distance at least $2r+1$, $G$ can be computed in time $T$ and there exists an algorithm that for every $w \in \pmi^d$ such that for some $z \in \pmi^k$, $\#_{\neq}(w,G(z))\leq r$ finds such $z$ in time $T$ (note that such $z$ is unique).


Given $[d,k,r,T]$ code $G$, for every $j \in [k]$, we define a function \equ{g_j(x) \doteq \max\left\{1-\frac{r}{2d}, \max_{w \in W_j} \la \bar{w}, x \ra \right\} \label{eq:def-g-code},}
where $W_j \doteq \{ G(z) \cond z \in \pmi^k, z_j=1\}$.
  As before, we note that $g_j$ is convex and $1$-Lipschitz (relative to $\ell_2$).
\iffull
\begin{thm}
\label{thm:l2-efficient}
Let $G$ be a $[d,k,r,T]$ code. Let $\K = \B_2^d$ and we define $\cH_G \doteq \{g_j \cond j \in [k]\}$ for $g_j$ defined in eq.~\eqref{eq:def-g-code}. Let $D$ be the uniform distribution over $\cH_G$. Then for every $x \in \K$, $g_j(x)$ can be computed in time $2T +O(d)$. Further,
for $n \leq k/2 $ and every set of samples $S \in \cH_G^n$ there exists an ERM $\hat{x}(S)$ such that
$$ F_{D}(\hat{x}(S)) - F^* \geq r/(4d)  .$$
\end{thm}
\begin{proof}
Let $W \doteq \{G(z) \cond z \in \pmi^k\}$. For every distinct $w^1,w^2 \in W$, $\la \bar{w}^1 , \bar{w}^2 \ra = 1-2\cdot \#_{\neq}(w^1,w^2)/d \leq 1-(2r+1)/d < 1-r/(2d)$. Therefore, by the definition of $g_j$, for every $w \in W$,  $g_j(\bar{w}) = 1$ if $w \in W_j$ and $g_j(\bar{w}) = 1-r/(2d)$ otherwise.
Now for $z \in \pmi^k$, let $\#_1(z)$ denote the number of indices $j \in [k]$, such where $z_j=1$. For $w = G(z)$ there are exactly $\#_1(z)$ indices $j$, such that $w \in W_j$. This means that for a function $g_j$ chosen randomly with respect to $D$, we have that $g_j(\bar{w}) = 1$ with probability exactly $\#_1(z)/k$. This implies that $F_{D}(\bar{w}) = \#_1(z)/k + (1-\#_1(z)/k) (1-r/(2d)) = 1- (1-\#_1(z)/k) \cdot r/(2d)$.

Let $S= (g_{j_1}, \ldots, g_{j_n})$ be any set of $n$ points from $\cH_G$. Observe that $\min_{\K}(F_S) = 1-r/(2d)$ and $F^* =\min_{\K}(F_D) = 1-r/(2d)$ (the minimum is achieved at the origin $\bar{0}$). Now, for $I = \{j_1,\ldots,j_n\}$ let $z^I$ denote the vector such that $z^I_j = -1$ if $j \in I$ and $z^I_j = 1$, otherwise. Clearly, $\#_1(z^I) = k-|I|\geq k-n \geq k/2$. Let $w^I \doteq G(z^I)$ and let $\hat{x}(S) \doteq \bar{w}^I$.

Observe that for all $i\in [n]$, $z^I_{j_i} = -1$ and therefore $g_{j_i}(\hat{x}(S)) = 1-r/(2d)$. This means that
$F_S(\hat{x}(S)) =  1-r/(2d)$ and therefore $\hat{x}(S)$ is a minimizer of $F_S$.  On the other hand, $F_D(\hat{x}(S))  =  1- (1-\#_1(z^I)/k) \cdot r/(2d) \geq 1-r/(4d)$. This implies the claimed generalization error of $r/(4d)$.

Finally we need to show that for any $x\in \K$, $g_j(x)$ can be computed in time $2T + O(d)$. We use the following simple algorithm, let $\sgn(x)$ denote the element-wise application of the sign function. We apply the decoding algorithm for $G$ to $\sgn(x)$ and $z$ denote the result. If the decoding, succeeds, $z_j = 1$ and $\la x , \overline{G(z)} \ra \geq 1-r/(2d)$, then output $g_j(x) = \la x , \overline{G(z)} \ra$. Otherwise, output $g_j(x) = 1-r/(2d)$. It is easy to see that the running time of this algorithm is $2T + O(d)$.

To prove the correctness, observe that we only need to output a value that is different from $1-r/(2d)$
when there exists $w$ such that $w \in W_j$ and $\la x, \bar{w} \ra > 1-r/(2d)$. If $\la x, \bar{w} \ra > 1-r/(2d)$ then
\equ{\|x - \bar{w}\|^2= \|x \|^2 +  \|\bar{w}\|^2 - 2 \la x, \bar{w} \ra \leq 2 - 2(1- r/(2d)) = r/d . \label{eq:x-close}} Now for every $i \in [d]$ such that $\sgn(x_i) \neq w_i$, we have that $(x_i  - \bar{w}_i)^2 \geq 1/d$. Therefore $$\|x - \bar{w}\|^2 = \sum_{i\in [d]}(x_i  - \bar{w}_i)^2 \geq \frac{\#_{\neq}(\sgn(x),w)}{d} .$$
Combining it with eq.\eqref{eq:x-close} we obtain that we only need to output value that is different from $1-r/(2d)$ only when there exists $w \in W_j$ such that $\#_{\neq}(\sgn(x),w) \leq r$. By the properties of $G$, in this case there is an algorithm that in time $T$ will find the unique $z$ such that $G(z) = w$.  Given such $z$ we can compute $G(z)$ in time $T$ and verify that $z_j = 1$ and $\la x , \overline{G(z)} \ra \geq 1-r/(2d)$ in time $O(d)$.
\end{proof}\fi

We can now use any existing constructions of efficient binary error-correcting codes to obtain a lower bound that uses only a small set of efficiently computable convex functions. Getting a lower bound that has asymptotically optimal dependence on $d$ requires that $k = \Omega(d)$ and $r = \Omega(d)$ (referred to as being {\em asymptotically good}). The existence of efficiently computable and asymptotically good binary error-correcting codes was first shown by Justesen \cite{Justesen:72}. More recent work of Spielman \cite{Spielman:96} shows existence of asymptotically good codes that can be encoded and decoded in $O(d)$ time. In particular, for some constant $\rho > 0$, there exists a $[d, d/2,\rho \cdot d ,O(d)]$ binary error-correcting code. As a corollary we obtain the following lower bound.

\begin{cor}
\label{cor:l2-efficient}
Let $G$ be an asymptotically-good $[d, d/2,\rho \cdot d ,O(d)]$ error-correcting code for a constant $\rho > 0$. Let $\K = \B_2^d$ and we define $\cH_G \doteq \{g_j \cond j \in [d/2]\}$ for $g_j$ defined in eq.~\eqref{eq:def-g-code}. Let $D$ be the uniform distribution over $\cH_G$. Then for every $x \in \K$, $g_j(x)$ can be computed in time $O(d)$. Further,
for $n \leq d/4$ and every set of samples $S \in \cH_G^n$ there exists an ERM $\hat{x}(S)$ such that
$$ F_{D}(\hat{x}(S)) - F^* \geq \rho/4.$$
\end{cor}

\iffull

\section{Bounded-Range Convex Optimization}
\label{sec:non-lipschitz}
As we have outlined in the introduction, SCO is solvable in the more general setting in which instead of the Lipschitz bound and radius of $\K$ we have a bound on the range of functions in the support of distribution. Recall that for a bound on the absolute value $B$ we denote this class of functions by $\F(\K,B)$. This setting is more challenging algorithmically and has not been studied extensively. For comparison purposes and completeness, we state a recent result for this setting from \cite{RakhlinS15} (converted from the online to the stochastic setting in the standard way).
\begin{thm}[\cite{RakhlinS15}]
Let $\K = \B_2^d(R)$ for some $R>0$ and $B>0$. There is an efficient algorithm $\A$ that given $\eps,\delta >0$ and $n=O(\log(B/\eps)\log(1/\delta) B^2/\eps^2))$ i.i.d.~samples from any distribution $D$ supported on $\F(\K,B)$
 outputs an $\eps$-optimal solution to $F_D$ over $\K$ with probability $\geq 1-\delta$.
\end{thm}
The case of general $\K$ can be handled by generalizing the approach in \cite{RakhlinS15} or using the algorithms in \cite{BelloniLNR15,FeldmanGV:15}. Note that for those algorithms the sample complexity will have a polynomial dependence on $d$ (which is unavoidable in this general setting).
\eat{
\begin{thm}[\cite{FeldmanGV:15}]
Let $\K\subseteq \R^d$ be a convex body given by a membership oracle $\B_2^d(R_0) \subseteq \K \subseteq \B_2^d(R_1)$. There is an algorithm that that given $\eps,\delta >0$ and $n=\tilde O(d^3\log(1/\delta) B^2/\eps^2))$ i.i.d.~samples from any distribution $D$ supported on $\F(\K,B)$ outputs an $\eps$-optimal solution with probability $\geq 1-\delta$. The algorithm runs in $\poly(d, 1/\eps, \log(R_1/R_0))$ time.
\end{thm}
}

In contrast to these results, we will now demonstrate that for such problems an ERM algorithm will require an infinite number of samples to succeed already for $d=2$. As in the proof of Thm.~\ref{thm:l2-smooth} we define $f_V(x) = \sum_{w\in V} \phi(\la w, x\ra)$. However we can now use the lack of bounds on the Lipschitz constant (or smoothness) to use $\phi(a)$ that is equal to $0$ for $a \leq 1-\alpha$ and $\phi(1)=1$. For every $m \geq 2$, we can choose a set of $m$ vectors $W$ evenly spaced on the unit circle such that for a sufficiently small $\alpha > 0$, $\phi(\la w, x\ra)$ will not interact with $\phi(\la w', x\ra)$, for any two distinct $w,w' \in W$. More formally, let $m$ be any positive integer, let $w^i \doteq (\sin(2\pi \cdot i/m),\cos(2\pi \cdot i/m))$ and let  $W_m \doteq \{w^i \cond i\in[m]\}$. Let
$$\phi_\alpha(a) \doteq  \left\{\arr{ll}{ 0 & \mbox{if }  a \leq 1-\alpha \\ (a-1+\alpha)/\alpha & \mbox{otherwise.}} \right.$$
For $V \subseteq W_m$ we set $\alpha \doteq 2/m^2$ and
define
\equ{f_V(x) \doteq \sum_{w\in V} \phi_\alpha(\la w,x \ra) \label{eq:def-h-circle}.}
\iffull
See Figure \ref{fig:nl} for an illustration.  It is easy to see that $f_V$ is convex. We now verify that the range of $f_V$ is $[0,1]$ on $\B_2^2$. Clearly, for any unit vector $w^i \in W_m$, and $x \in \B^2_2$, $\la w^i,x\ra \in [-1,1]$ and therefore $\phi_\alpha(\la w^i, x\ra)\in [0,1]$. Now it suffices to establish that for every $x \in \B_2^2$, there exists at most one vector $w\in W_m$ such that $\phi_\alpha(\la w, x\ra) >0$.
To see this, as in Lemma \ref{lem:smoothness}, we note that if $\phi_\alpha(\la w, x\ra) >0$ then $\la w, x\ra > 1-\alpha$. For $w\in W_m$ and $x\in \B_2^2$, this implies that $\|w - x\| < \sqrt{1 + 1 - 2(1-\alpha)} = \sqrt{2\alpha}$. For our choice of $\alpha=2/m^2$, this implies that $\|w - x\| < 2/m$. On the other hand, for $i\neq j \in [m]$, we have $$\|w^i - w^j\| \geq \|w^1 - w^m\| \geq \sin(2\pi/m) \geq 2\pi/m - (2\pi/m)^3/6 \geq 4/m .$$
Therefore there does not exist $x$ such that $\phi_\alpha(\la w^i, x\ra) >0$ and $\phi_\alpha(\la w^j, x\ra) >0$. Now we can easily establish the lower bound.
\fi
\begin{thm}
\label{thm:l2-circle}
Let $\K = \B_2^2$ and $m\geq 2$ be an integer. We define $\cH_m \doteq \{f_V \cond V \subseteq W\}$ for $f_V$ defined in eq.~\eqref{eq:def-h-circle}. Let $D_m$ be the uniform distribution over $\cH_m$. Then for $n \leq \log m$ and every set of samples $S$ there exists an ERM $\hat{x}(S)$ such that $$\pr_{\bS \sim D_m^n}\lb  F_{D}(\hat{x}(\bS)) - F^*   \geq 1/2 \rb > 1/2. $$
\end{thm} \iffull
\begin{proof}
Let $\bS= (f_{\bV_1}, \ldots, f_{\bV_n})$ be the random samples. Clearly, $F^* = 0$ and
$\min_{\K}(F_\bS) = 0$. Further, the analysis above implies that for every $w \in W_m$ and $V \subseteq W_m$, $f_V(w) = 1$ if $w\in V$ and $f_V(w) = 0$ otherwise. Hence $F_{D_m}(w) = 1/2$. Now, if $\bigcup_{i \in [n]}\bV_i \neq W_m$ then let $\hat{x}(\bS) \doteq w$ for any $w\in W_m\setminus \bigcup_{i \in [n]}\bV_i$. Then for all $i$, $h_{\bV_i}(\hat{x}(\bS)) = 0$ and hence $F_\bS(\hat{x}(\bS)) = 0$. This means that $\hat{x}(\bS)$ is a minimizer of $F_\bS$ and $F_{D_m}(\hat{x}(\bS)) - F^* = 1/2$.

Now, exactly as in Thm.~\ref{thm:l2-max}, we can conclude that $\bigcup_{i \in [n]}\bV_i = W_m$ with probability at most $$\lp1-2^{-n}\rp^{m} \leq e^{-2^{-n} \cdot m} \leq e^{-1} < \fr{2}.$$
\end{proof} \fi
This lower bound holds for every $m$. This implies that the sample complexity of $1/2$-optimizing $\F(\B_2^2,1)$ over $\B_2^2$ is infinite.
\fi
\section{Discussion}
Our work points out to substantial limitations of the classic approach to understanding and analysis of generalization in the context of general SCO.  Further, it implies that in order to understand how well solutions produced by an optimization algorithm generalize, it is necessary to examine the optimization algorithm itself. This is a challenging task that we still have relatively few tools to address. Yet such understanding is also crucial for developing theory to guide the design of optimization algorithms that are used in machine learning applications.

One way to bypass our lower bounds is to use additional structural assumptions. For example, for generalized linear regression problems uniform convergence gives nearly optimal bounds on sample complexity \cite{KakadeST:08}. One natural question is whether there exist more general classes of functions that capture most of the practically relevant SCO problems and enjoy dimension-independent (or, scaling as $\log d$) uniform convergence bounds.

An alternative approach is to bypass uniform convergence (and possibly also ERM) altogether. Among a large number of techniques that have been developed for ensuring generalization, the most general ones are based on notions of stability \cite{BousquettE02,ShwartzSSS10}. However, known analyses based on stability often do not provide the strongest known generalization guarantees (\eg high probability bounds require very strong assumptions). Another issue is that we lack general algorithmic tools for ensuring stability of the output. Therefore many open problems remain and significant progress is required to obtain a more comprehensive understanding of this approach. Some encouraging new developments in this area are the use of notions of stability derived from differential privacy \cite{DworkFHPRR14:arxiv,DworkFHPRR15:arxiv,BassilyNSSSU16} and the use of techniques for analysis of convergence of convex optimization algorithms for proving stability \cite{HardtRS16}.

\section*{Acknowledgements}
I am grateful to Ken Clarkson, Sasha Rakhlin and Thomas Steinke for discussions and insightful comments related to this work.

\iffull
\bibliographystyle{alpha}	
\bibliography{vf-allrefs-central}
\else
\bibliographystyle{abbrv}	
\small{
\bibliography{vf-allrefs-central}
}
\fi
\end{document}